\documentclass{article}

\PassOptionsToPackage{numbers, compress}{natbib}

     \usepackage[preprint]{neurips_2024}

\newif\ifnatbib
\natbibfalse
\ifnatbib
\usepackage{neurips_2024}
\else

\usepackage[utf8]{inputenc} 
\usepackage[T1]{fontenc}    
\usepackage{url}            
\usepackage{booktabs}       
\usepackage{amsfonts}       
\usepackage{nicefrac}       
\usepackage{microtype}      

\usepackage[dvipsnames]{xcolor}      
\usepackage[utf8]{inputenc}
\usepackage[T1]{fontenc}    
\usepackage[colorlinks=true, linkcolor=BrickRed, urlcolor=blue, citecolor=Blue, anchorcolor=blue, backref=page]{hyperref}
\renewcommand*{\backref}[1]{}
\renewcommand*{\backrefalt}[4]{%
    \ifcase #1%
          \or [Cited on page~#2.]%
          \else [Cited on pages~#2.]%
    \fi%
    }
\usepackage{amsmath}
\usepackage{amsthm}
\usepackage{amssymb}
\usepackage{bm}
\usepackage{url}            %
\usepackage{booktabs}
\usepackage{amsfonts} 
\usepackage{comment}
\usepackage{tikz}
\usetikzlibrary{shapes.geometric, arrows, shadows, bayesnet}
\usepackage{subcaption}
\usepackage{xcolor}

\usepackage{enumitem}
\usepackage{tabularx}

\numberwithin{equation}{section}
\renewcommand{\mathbf}[1]{{\bm{#1}}}
\newcommand{\ab}{\mathbf{a}}
\newcommand{\bb}{\mathbf{b}}

\newcommand{\fb}{\mathbf{f}}
\newcommand{\gb}{\mathbf{g}}
\newcommand{\hb}{\mathbf{h}}

\newcommand{\vb}{{\bm{v}}}

\newcommand{\xb}{\mathbf{x}}
\newcommand{\yb}{\mathbf{y}}
\newcommand{\zb}{\mathbf{z}}
\renewcommand{\sb}{\mathbf{s}}

\newcommand{\Cb}{\mathbf{C}}

\newcommand{\Sb}{\mathbf{S}}

\newcommand{\Ub}{\mathbf{U}}
\newcommand{\Vb}{\mathbf{V}}
\newcommand{\Wb}{\mathbf{W}}
\newcommand{\Xb}{\mathbf{X}}
\newcommand{\Yb}{\mathbf{Y}}
\newcommand{\Zb}{\mathbf{Z}}

\newcommand{\Dcal}{\mathcal{D}}

\newcommand{\Gcal}{\mathcal{G}}

\newcommand{\Mcal}{\mathcal{M}}
\newcommand{\Ncal}{\mathcal{N}}

\newcommand{\Pcal}{\mathcal{P}}

\newcommand{\Xcal}{\mathcal{X}}
\newcommand{\Ycal}{\mathcal{Y}}
\newcommand{\Zcal}{\mathcal{Z}}

\newcommand{\EE}{\mathbb{E}} 
\newcommand{\RR}{\mathbb{R}} 
\newcommand*{\Psib}{\bm{\Psi}}

\newcommand*{\mub}{\bm{\mu}}
\newcommand*{\nub}{\bm{\nu}}

\newcommand*{\phib}{\bm{\phi}}

\usepackage{amsthm,thmtools,thm-restate}
\usepackage{cleveref}
\RequirePackage{amsmath}
\RequirePackage{amssymb}
\RequirePackage{mathtools}
\ifx\proof\undefined 
\RequirePackage{amsthm}
\fi
\crefname{figure}{Fig.}{Figs.}
\crefname{definition}{Defn.}{Defns.}
\crefname{corollary}{Cor.}{Cors.}
\crefname{proposition}{Prop.}{Props.}
\crefname{theorem}{Thm.}{Thms.}
\crefname{remark}{Remark}{Remarks}
\crefname{principle}{Principle}{Principles}
\crefname{lemma}{Lemma}{Lemmata}
\crefname{claim}{Claim}{Claims}
\crefname{table}{Tab.}{Tabs.}
\crefname{section}{\S}{\S\S}
\crefname{subsection}{\S}{\S\S}
\crefname{subsubsection}{\S}{\S\S}
\crefname{assumption}{Asm.}{Asms.}
\crefname{appendix}{Appx.}{Appx.}
\crefname{equation}{Eq.}{Eqs.}
\crefname{example}{Example}{Examples}

\ifx\BlackBox\undefined
\newcommand{\BlackBox}{\rule{1.5ex}{1.5ex}}  
\fi

\ifx\QED\undefined
\def\QED{~\rule[-1pt]{5pt}{5pt}\par\medskip}
\fi

\ifx\proof\undefined
\newenvironment{proof}{\par\noindent{\bf Proof\ }}{\hfill\BlackBox\\[2mm]}
\fi
\ifx\proofsketch\undefined
\newenvironment{proofsketch}{\par\noindent{\bf Proof Sketch\ }}{\hfill\BlackBox\\[2mm]}
\fi

\theoremstyle{plain} 
\ifx\theorem\undefined
\newtheorem{theorem}{Theorem}
\numberwithin{theorem}{section}
\fi
\ifx\property\undefined
\newtheorem{property}{Property}
\fi
\ifx\corollary\undefined
\newtheorem{corollary}[theorem]{Corollary}
\fi
\ifx\lemma\undefined
\newtheorem{lemma}[theorem]{Lemma}
\fi
\ifx\proposition\undefined
\newtheorem{proposition}[theorem]{Proposition}
\fi
\ifx\assum\undefined

\fi

\theoremstyle{definition} 
\ifx\definition\undefined
\newtheorem{definition}[theorem]{Definition}
\fi
\ifx\assum\undefined
\newtheorem{assumption}[theorem]{Assumption}
\fi

\theoremstyle{remark} 
\ifx\remark\undefined
\newtheorem{remark}[theorem]{Remark}
\fi
\ifx\example\undefined
\newtheorem{example}[theorem]{Example}
\fi
\ifx\lemma\undefined
\newtheorem{lemma}[theorem]{Lemma}
\fi
\ifx\conjecture\undefined
\newtheorem{conjecture}[theorem]{Conjecture}
\fi
\ifx\fact\undefined
\newtheorem{fact}[theorem]{Fact}
\fi
\ifx\claim\undefined
\newtheorem{claim}[theorem]{Claim}
\fi
\ifx\assum\undefined

\fi

\def\independenT#1#2{\mathrel{\rlap{$#1#2$}\mkern2mu{#1#2}}}

\DeclareMathOperator*{\Cov}{Cov}

\newcommand{\X}[0]{\mathbf{X}}
\newcommand{\Z}[0]{\mathbf{Z}}
\newcommand{\U}[0]{\mathbf{U}}

\newcommand{\Y}[0]{\mathbf{Y}}

\newcommand{\V}[0]{\mathbf{V}}

\newcommand{\x}[0]{\mathbf{x}}
\newcommand{\y}[0]{\mathbf{y}}
\newcommand{\z}[0]{\mathbf{z}}
\newcommand{\w}[0]{\mathbf{w}}
\newcommand{\lct}[0]{\mathbf{t}}

\newcommand{\W}[0]{\mathbf{W}}
\newcommand{\DF}[0]{\mathtt{DeconFlow}}

\newcommand{\doo}[0]{\mathrm{do}}

\newcommand{\indep}{\perp \!\!\! \perp}

\newcommand{\patrick}[1]{{\textcolor{red}{#1 [PB].}}}
\newcommand{\michel}[1]{{\textcolor{blue}{#1 [MB].}}}
\newcommand{\frederick}[1]{{\textcolor{orange}{#1 [FE].}}}

\usepackage{wrapfig}
\usepackage{algorithm}
\usepackage{algorithmic}

\ifnatbib
\else
\fi

\title{Controlling for discrete unmeasured confounding in nonlinear causal models}

\author{%
Patrick Burauel \\
  California Institute of Technology\\
  Pasadena, CA 91125, USA \\
  \texttt{pburauel@caltech.edu} \\
  \And
  Frederick Eberhardt\\ 
  California Institute of Technology\\
  Pasadena, CA 91125, USA \\
  \texttt{fde@caltech.edu} \\
  \And
  Michel Besserve  \\ 
  Max Planck Institute for Intelligent Systems\\
  72076 T\"{u}bingen, Germany \\
  \texttt{michel.besserve@tuebingen.mpg.de} \\  
}

\usepackage{hyperref}
\begin{document}

\maketitle

\begin{abstract}%
    Unmeasured confounding is a major challenge for identifying causal relationships from non-experimental data. Here, we propose a method that can accommodate unmeasured discrete confounding. Extending recent identifiability results in deep latent variable models, we show theoretically that confounding can be detected and corrected under the assumption that the observed data is a piecewise affine transformation of a latent Gaussian mixture model and that the identity of the mixture components is confounded. We provide a flow-based algorithm to estimate this model and perform deconfounding. Experimental results on synthetic and real-world data provide support for the effectiveness of our approach. 
\end{abstract}
\section{Introduction}

One of the fundamental challenges of causal inference is the separation of the causal effect from confounding, that is, from statistical dependencies that arise from common causes of the candidate cause and effect. In Pearl's notation \cite{Pearl_2009}, this difference is captured by the key contrast between the merely predictive conditional probability $P(Y|X)$ and the causal effect $P(Y|\doo(X))$. When confounding variables are observed, confounding can be controlled for by a variety of covariate adjustment techniques \citep{imbens15,chernozhukov2018double}. The ability to also deconfound the causal effect in the case of \emph{unobserved} confounding is one of the motivations for the use of randomized controlled trials. The challenge of how to deconfound the causal effect \emph{without experimentation} has given rise to a variety of approaches that require different assumptions for identification. These include instrumental variable approaches \cite{imbens15}, approaches based on parametric assumptions (such as in additive noise models \cite{tashiro14, hoyer2008estimation}, linear models \cite{JS, JSb} or binary Gaussian mixture models \cite{gordon2023causal}), or settings where observed confounding is assumed to be representative of unobserved confounding \cite{cinelli20sensitivity}.

In this paper, we contribute to the effort to address unmeasured confounding in purely observational settings by imposing restrictions on the model class. Unlike previous work, we do this by reformulating a confounded cause-effect model as an equivalent latent variable model with a Gaussian mixture prior (see Figure~\ref{fig_models}). We then leverage the results in \cite{kivva2022identifiability} that assure identification (up to an affine transformation) of the latent Gaussian mixtures under the assumption of a piecewise affine mapping between latent and observed variables. We show that further constraints on this model specific to our setting (notably causal order) allow to identify causal effects despite (discrete) unobserved confounding. Implementing this approach with a flow-based deep generative model, 
we show on both synthetic and real data how to estimate the desired causal effects despite unmeasured confounding.
\begin{figure}
    \centering 
\includegraphics[width=1\linewidth]{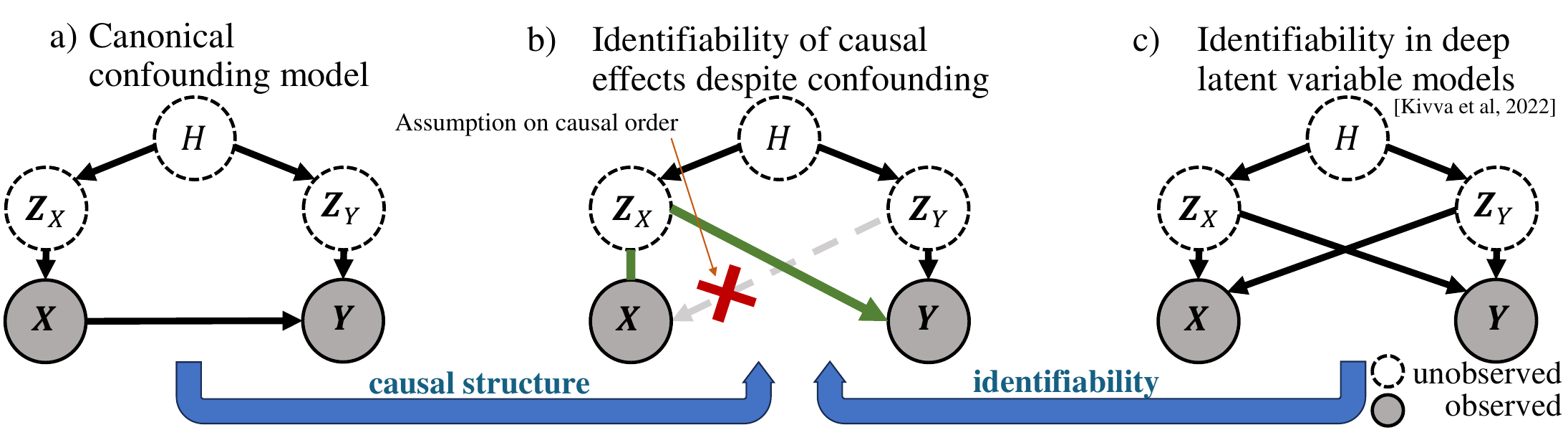}
    \caption{On the left, $\X$ causes $\Y$ and is confounded by $H$. On the right, observed variables $\W=(\X,\Y)$ are generated by latent variables $\Z$, whose identifiability up to affine transformation under model restrictions is shown by \cite{kivva2022identifiability}. We combine knowledge of causal structure with identifiability results for latent variable models to estimate causal effects despite unmeasured confounding (middle).\label{fig_models}}
\end{figure}

\textbf{Notations.} We will use uppercase letters for random variables (e.g. $X$) and lowercase for deterministic ones (e.g. a realization $x$ of $X$). Functions and variables that may be vector-valued will be denoted in bold (e.g. $\X,\,\fb$,~...), and $^\top$ denotes transposition. We will use non-bold capital letters for (deterministic) matrices, e.g. $A$. $P(.)$ denotes a probability distribution, while $p(.)$ denotes the corresponding density with respect to the Lebesgue measure.

\section{Background}\label{sec_background}
\textbf{Canonical cause-effect model in causal inference.} In causal inference, the canonical cause-effect model ``$\Xb$ causes $\Yb$'' can be represented by a pair of so-called \textit{structural equations} \cite{Pearl_2009}: 
\begin{equation}\label{eq_causeeffect}
\X \coloneqq \fb_X(\Z_X)\,,\quad 
\Y \coloneqq \fb_Y(\X, \Z_Y)\,,\quad \mbox{with} \quad
(\Z_X,\Z_Y) \sim P_Z(\Z_X,\Z_Y)\,,
\end{equation}
where the exogenous variables $(\Z_X,\Z_Y)$ are idiosyncratic error terms representing the influence of external factors on the system, and $(\fb_X,\fb_Y)$ are the causal mechanisms associated to each variable. Causal effects of interests are entailed by the mechanism $\fb_Y$ that describes the influence of $\X$ on $\Y$. 
Confounding then posits the existence of a common cause $H$ that influences both idiosyncratic error terms, 
such that they become dependent when marginalizing with respect to $H$, leading to
\[
\textstyle P_Z(\Z_X,\Z_Y)=\sum_h P(\Z_X|H=h)P(\Z_Y|H=h)P(H=h) \neq P_{\Z_X}(\Z_X)P_{\Z_Y}(\Z_Y)\,,\] 
as depicted in the causal diagram of Figure~\ref{fig_models}a. Accounting for this dependence is necessary for the unbiased estimation of the causal effect but is difficult as $\Z_X$, $\Z_Y$ and $H$ are typically unobserved.\footnote{We provide a brief description of the formalism of structural causal models in Appendix \ref{app_SCM}.} 

\textbf{Identifiability of latent variable models.} The field of \textit{latent variable models} (LVM) \cite{kingma2019introduction,papamakarios2021normalizing} addresses the learnability of models mapping latent variables $\Z$ to observations $\W$ using a so-called mixing function $\Psib$ such that $\W=\Psib(\Z)$, using only samples from the observation distribution $P(\W)$. Identifiability results provide guaranties that, given infinite data, the ground truth $(\Psib,\Z)$ can be recovered from $P(\W)$ in the large sample limit, up to well-characterized ambiguities. We build on results presented by \cite{kivva2022identifiability}, who consider a generative model for observed variables $\W$ of the form:
\begin{align*}
H &\sim \text{Cat}(K_H, \mathbf{\pi})\,, \\
\Z \mid H = h &\sim \mathcal{N}(\mathbf{\mu_h}, \Sigma_h)\,, \\
\W &= \Psib(\Z) ,
\end{align*}
where $\mbox{Cat}(K,\mathbf{\pi})$ denotes a categorical distribution with $K$ categories and an associated vector of event  probabilities $\mathbf{\pi}$. 
Assuming that $\Psib$ is a piecewise affine injective function (which can be implemented by ReLU networks), \cite{kivva2022identifiability} show identifiability of $\Psib$ and $\Z$ up to an affine transformation \cite[Theorem 3.2]{kivva2022identifiability}. This model is depicted in Figure~\ref{fig_models}c.

\section{Theoretical framework for discrete decounfounding}
\label{sec_discrete_conf_model}
\subsection{General setting}

\textbf{Mapping cause-effect models to LVMs.} 
We consider the above cause-effect model in a setting where an observed $n$-dimensional vector $\X$ causes an observed $m$-dimensional effect vector $\Y$, and where, as commonly assumed, exogenous variables have matching dimensions, i.e. $\Z_X\in \RR^n$ and $\Z_Y\in \RR^m$.\footnote{The special cases of scalar cause and/or effect are included.} 
 We explore the idea that exogenous variables $\Z_X,\Z_Y$ and mechanisms $\fb_X,\fb_Y$ can be used to construct a corresponding LVM, from which we can then leverage the identifiability results to address unmeasured confounding. The key ideas are the following: We can replace the generative mechanism of $\Yb$ based on $\Xb$ by one based on $\Zb_1$ by rewriting
\begin{equation}\label{eq_model}
\Yb\coloneqq \fb_Y(\X,\Zb_Y)= \fb_Y\left(\fb_X(\Z_X),\Zb_Y\right)\triangleq \Psib_Y(\Zb_X,\Zb_Y).    
\end{equation}
If we additionally introduce $\Psib_X(\Z_X,\Z_Y)\triangleq \fb_X(\Z_X)$ and concatenate the exogenous variables into the latent vector $\Z\!=\!(\Z_X,\Z_Y)$, we can build a well-defined mapping $\Psib:\RR^{m+n}\mapsto\RR^{m+n}$ from exogenous latent variables to observed variables $\W\!=\!\left(\X,\Y\right)$ such that $\Psib(\Z)\!=\!\left(\Psib_X(\Z),\Psib_Y(\Z)\right)$. This  corresponds to the LVM diagram of Figure~\ref{fig_models}c. Analogous to the causal model in Figure~\ref{fig_models}a, confounding is induced by a latent variable $H$ that causes both $\Z_X$ and $\Z_Y$.   

\textbf{Leveraging LVM identifiability to address confounding.} 
Concretely, to connect LVM identifiability to causal deconfounding, we introduce the following assumptions on the cause-effect model.

\begin{assumption}\label{assum:partinv} The function $\fb_Y: \RR^n \times \RR^m \to \RR^m$ is  Continuous Deterministic Piecewise Affine (CDPA)\footnote{CDPA functions can be easily implemented  by feedforward neural networks with ReLU activation functions.}
 and for all $\xb\in \RR^n$, $\zb_Y\mapsto \fb_Y(\xb,\zb_Y)$ is injective. 
\end{assumption}
Additionally, we make an assumption about the relation between $\Z_X$ and $\X$:
\begin{assumption}\label{assum:phiinject}
$\fb_X:\RR^n\to \RR^n$ is CDPA and invertible. 
\end{assumption}
In combination, these two assumptions will ensure the mapping $\Psib$ belongs to the function class analyzed in \cite{kivva2022identifiability}. The final key to identifiability is a Gaussian mixture model of the exogenous variables and their confounding induced by $H$. 
\begin{assumption}\label{assum:nondeg}
The exogenous variables are generated according to the following model:
\begin{align}
    H &\sim \mathrm{Cat}(K_H,\boldsymbol{\pi})\,,\\ \label{eq_LgivenH_QgivenH}
    L|H & \sim \mathrm{Cat}(K_L,p(L|H))\,, \quad \quad \quad
    Q|H  \sim \mathrm{Cat}(K_Q,p(Q|H))\,,\\
    \Zb_X|L\!=\!l&\sim \mathcal{N}(\mub_{l},\Sigma_{l}^X)\,, \,\quad \qquad \hspace{2mm}
    \Zb_Y|Q\!=\!q \sim \mathcal{N}(\nub_{q},{\Sigma_{q}^Y})\,,\label{eq_GMM_assumption}
\end{align}
where at least one mixture component $l$ that occurs with non-zero probability has $\Sigma_l^X$ positive definite.
\end{assumption}

Note that, without loss of generality, we make the separation of the effect of $H$ on the cause vs.\ the effect side explicit with Eq. \eqref{eq_LgivenH_QgivenH}. 
We now turn to proving that this model setup and the discussed assumptions allow us to identify causal quantities.

\subsection{Identifiability} 

\begin{restatable}{theorem}{identif}\label{thm:identif}
    Under Assumptions \ref{assum:partinv}, \ref{assum:phiinject}, and \ref{assum:nondeg} the mixture components and the causal mechanism for the effect $(\Zb_Y,\fb_Y)$ in Eq.~\eqref{eq_model} is identifiable up to an invertible affine reparameterization of $\Zb_Y$. More precisely, let $(\tilde{\Zb}_Y,\tilde{\fb}_Y)$ be the latent variable and mechanism obtained by fitting the model to the observation distribution $P(\X,\Y)$, then we have, for some $(m\times m)$ invertible matrix $S$ and some $(m\times 1)$ vector $\bb$
    \[
    \fb_Y(\xb,\zb_Y)=\tilde{\fb}_Y(\xb,S \zb_Y+\bb)\,, \quad\mbox{and} \quad\tilde{\Zb}_Y = S \Zb_Y+\bb
    \,.
    \]   
\end{restatable}
\begin{proof}[Sketch of the proof (see Appendix~\ref{app_proofs} for the complete version).]
We will consider a latent variable model solution $\tilde{\Psib}:\Zb\to \W$ satisfying all assumptions and fitting the observational distribution $P(\X,\Y)$ perfectly. We study its relationship to the corresponding ground truth mapping ${\Psib}$ which generates the observations. This will then be linked to the cause-effect model solution $\tilde{\fb}_Y$ and its associated  ground truth model $\fb_Y$. The demonstration can be decomposed into three parts:

\textbf{(1)} The identifiability theory in \cite[Theorem 3.2]{kivva2022identifiability} implies that the latents $\Z$ can be recovered up to an affine transformation; more formally, the map $\tilde{\Psib}^{-1} \circ \Psib$ associating ground truth latents $\Z$ to recovered ones $\tilde{\Z}$ is an affine transformation with its linear map represented by a square matrix $A$. In addition, the constraint on the causal order enforces that $\Psib_X$ is not dependent on $\Z_Y$, which imposes a block triangular structure on $A$, encoding that the true $\Z_Y$ does not influence the recovered $\tilde{\Z}_X$. 

\textbf{(2)}  By Assumption~\ref{assum:nondeg} the mixture components' cross-covariance matrices between $\Zb_X$ and $\Zb_Y$ coordinates is zero for both the ground truth $\Z$ and recovered $\tilde{\Z}$. Identification up to affine transformation and permutation of these mixture components further constrains the relation between ground truth and recovered latents by forcing the matrix $A$ to be block diagonal.

\textbf{(3)} The final relation between ground truth and recovered cause-effect model is deduced from the shared structure of $\tilde{\Psib}$ and $\Psib$, and the block diagonality of $A$. 
\end{proof}

Note that the results by \cite{kivva2022identifiability} alone, allow the ambiguity of the identifiability results to be a general affine transformation without any restriction, which precludes the separation of the causal and the confounded variation in the observed $\Y$ and consequently the identification of the causal effect.

Provided the data generating process fits our assumptions, then our result guarantees that, in the infinite sample limit, we retrieve the ground truth causal mechanism up to some ambiguities. We now show that these remaining ambiguities do not affect our ability to estimate causal quantities such as the average treatment effect.

\textbf{Estimation of causal effects.} 
We now show that Theorem~\ref{thm:identif} implies that the average treatment effect is identifiabile, even though $P(L,H,Q)$ may remain unidentified. Given the graph in Figure~\ref{fig_models}b, we can see that $\Zb_Y$ satisfies the backdoor criterion \cite{Pearl_2009}, such that we can estimate the following interventional quantities by the adjustment formula: 
\begin{equation}
\EE \left[\Yb|\doo(\Xb=\xb) \right]=\int \yb\, p\left(\yb|\doo(\Xb=\xb)\right) d\yb =\int\!\!\int  \yb\, p\left(\yb|\Xb=\xb,\zb_Y\right) d\zb_Y d\yb \,.    
\end{equation}
That is, Theorem~\ref{thm:identif} provides the basis to deconfound the causal effect: 
\begin{restatable}{proposition}{identadjust}\label{prop:identadjust}
    Under the assumptions of Theorem~\ref{thm:identif}, assume additionally strict positivity of $p(\xb,\zb_Y)$ for almost all $\zb_Y$. Then, for any $\xb$ in the support of $P(\Xb)$, $\EE \left[\Yb|do(\Xb=\xb) \right]$ is identifiable from the observation of $P(\Xb,
    \Yb)$ with adjustment formula 
    \begin{equation}\label{eq_adjustment_main_text}
\EE \left[\Yb|do(\Xb=\xb) \right]=\EE_{\Zb_Y\sim P(\Zb_Y)}\left[\tilde{\fb}_Y(\xb,S \Zb_Y+\bb)\right]=\EE_{\tilde{\Zb}_Y\sim P(\tilde{\Zb}_Y)}\left[\tilde{\fb}_Y(\xb,\tilde{\Zb}_Y)\right]\,,
\end{equation}
where $P(\tilde{\Zb_Y})$ and $\tilde{\fb}_Y$ is the solution identified in Theorem~\ref{thm:identif}.
\end{restatable}
See Appendix~\ref{app_proofs} for the proof.
Importantly, we cannot rely on $\Zb_1$ as an adjustment variable, as it violates positivity by construction of our model (it is deterministically related to $\Xb$), in line with the point made by \cite{damour19multi}. Positivity of $p(\xb,\zb_Y)$ is achieved under mild assumptions: it only requires the occurrence of one non-degenerate mixtures component of $\Zb$ in the observational setting. 
\begin{restatable}{proposition}{sufficident}\label{prop:sufficident}
    If there exists $(l,q)$ such that $P(L=l,Q=q)>0$ and both $\Sigma_l^X$ and $\Sigma_q^Y$ are positive definite, then the positivity assumption on $p(\xb,\zb_Y)$ in Proposition~\ref{prop:identadjust} is satisfied. 
\end{restatable}
See Appendix~\ref{app_proofs} for the proof. 
Overall, the positive definite assumptions required on covariance matrices in Theorem~\ref{thm:identif} and Proposition~\ref{prop:sufficident} emphasize the importance of having independent (Gaussian) noise injected in both mechanism $\fb_X$ and $\fb_Y$ for identification.

\begin{figure}
    \centering 
    \includegraphics[width=.9\linewidth, trim=.2cm 11cm 1.8cm .8cm, clip]{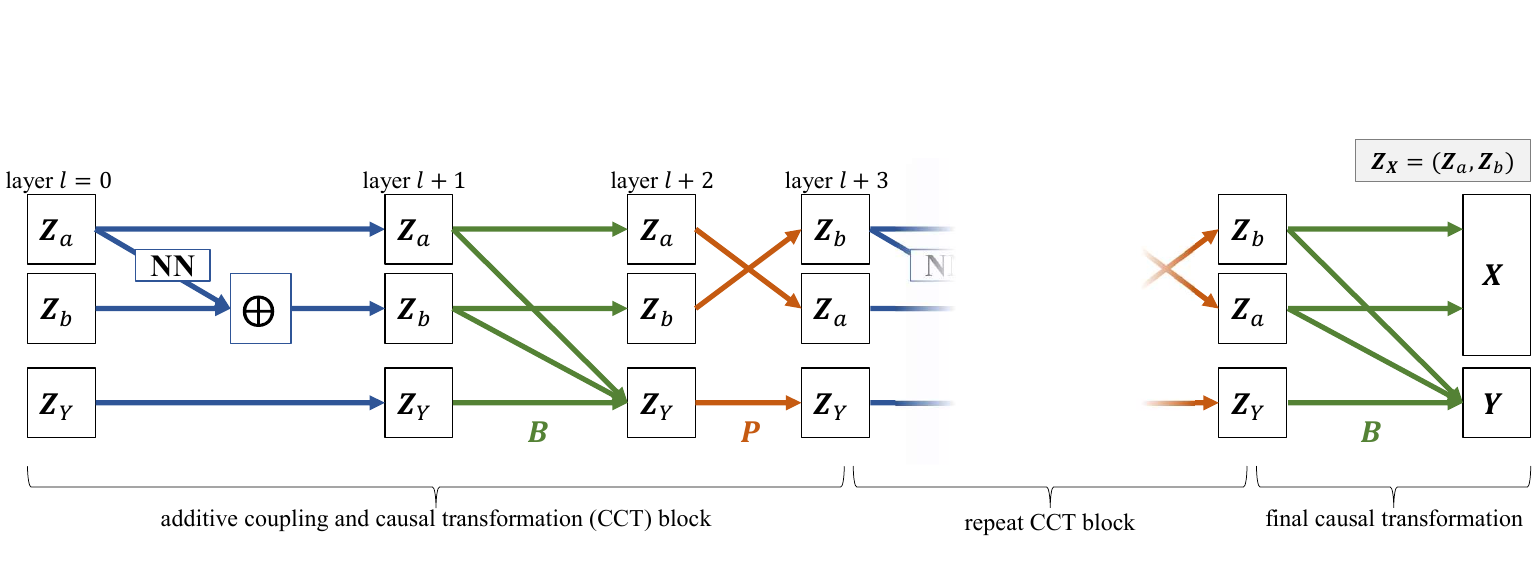}
    \caption{(Flow model implementation) The sequence of transformations that make up one block are composed of an additive coupling bijection from layer $l$ to $l+1$, see lines \texttt{5} and \texttt{6}, a causal transformation with a partly-diagonal structure ($\Z_Y$ node does not influence other nodes), see line \texttt{7}, from $l+1$ to $l+2$, and a permutation layer from $l+2$ to $l+3$. Line numbers refer to Algorithm \ref{alg_transformation_block}.}
    \label{fig_flow_structure}
\end{figure}

\section{Flow-based implementation}

We use flow-based models \citep{papamakarios2021normalizing} to estimate the discrete confounding model. Such models learn the (possibly complex) distribution of observed data by using successive transformations of a simpler base distribution. The trained model can then be used to sample from the data distribution. This generative aspect of flow-based models lends itself to our deconfounding application as it allows us to sample from $P(\tilde{\Zb}_Y)$, which is the latent variable that blocks the backdoor path and is used in Eq.~(\ref{eq_adjustment_main_text}). Unlike other generative models such as Variational Autoencoders, flow-based models allow optimization of the exact likelihood of the data, which seems to be critical for their use to estimate causal quantities precisely. Variational Autoencoders with a Gaussian mixture prior \cite{jiang2017variational}, as used in experimental section of \cite{kivva2022identifiability}, have proven not to perform as well as flow-based models for the application at hand.\footnote{We have implemented VAEs with appropriate architectural restrictions in experiments (not reported here) that did not exactly recover the true causal effects even in the simple $m=n=1$ linear case.}

\begin{wrapfigure}{R}{0.43\textwidth}
\begin{minipage}{0.41\textwidth}
\begin{algorithm}[H]
\caption{One $\DF$ transformation block, from layer $l$ to $l+3$}
\label{alg_transformation_block}
\begin{algorithmic}[1]
\STATE \textbf{Input:} $\z^{(l)}$
\STATE \textbf{Output:} $\z^{(l+3)}$
\STATE $\z_X^{(l)}, \z_Y^{(l)} \gets \text{split}(\z^{(l)})$
\STATE $\z_a^{(l)}, \z_b^{(l)} \gets \text{split}(\z_X^{(l)})$
\STATE $\lct^{(l)} \gets f_t(\z_a^{(l)})$
\STATE $\z_b^{(l+1)} \gets \z_b^{(l)} + \lct$ \\ \hfill \text{(additive coupling)}
\STATE $\z^{(l+2)} \gets \mathbf{B} \z^{(l+1)}$ \\\hfill \text{(causal transform: $\z_X \rightarrow \z_Y$)}
\STATE $\z_X^{(l+3)} \gets \mathbf{P} \z_X^{(l+2)}$
\STATE $\z_Y^{(l+3)} \gets \z_Y^{(l+2)}$
\end{algorithmic}
\end{algorithm}
\end{minipage}
\end{wrapfigure}

In flow-based models, observed variables $ \mathbf{w} :=(\mathbf{x}, \yb)\in \mathbb{R}^{m+n}$ are expressed as a transformation $T$ of $\mathbf{z}$, 
$\w = T(\z)$, sampled from a base distribution $p(\mathbf{z})$. Requiring $T$ to be differentiable and invertible licences the use of the change of variables formula to express the log-likelihood of the data as $\log p_\w(\w) = \log p_\z(\z) + \log|\det J_T(\z)|^{-1} $ or, using that $\z = T^{-1}(\w)$ and swapping inverse and determinant,
\begin{equation}
    \log p_\w(\w) = \log p_\z(T^{-1}(\w)) + \log |\det J_{T^{-1}}(\w)|.
\end{equation} The log-likelihood of the data can thus be expressed by evaluating the base distribution at the transformed $\w$ and accounting for the resulting change in volume by adding the log determinant of the inverse Jacobian of that transformation. To represent the Gaussian mixture structure of the latent variables in our generative model, see Eq.~\eqref{eq_GMM_assumption}, we use a Gaussian mixture model as a base distribution.\footnote{A GMM base distibution in flow-based models has previously been used by e.g. \cite{stimper2022resampling}.} The GMM is characterized by mixture weights ($\pi_k$), means ($\mub_k$) and covariances ($\Sigma_k$):
\begin{equation}\label{eq_GMM}
\textstyle
	p(\mathbf{z}) = \sum_{k=1}^{K} \pi_k \mathcal{N}(\mathbf{z}; \mub_k, \Sigma_k),
\end{equation}
where $K$ is the number of mixture components, $\pi_k$ are the mixture weights, and $\mathcal{N}(\mathbf{z}; \mu_k, \Sigma_k)$ with diagonal covariance matrix denotes the Gaussian distribution for component $k$.

In our causal inference setting, only transformations that respect the causal order of observed variables $\w$ are admissible. To ensure that information flows only in the causal direction from $\x$ to $\yb$, we need to restrict the transformations to be lower-triangular. We first introduce a simple one-layer, linear flow, which allows us to introduce the required restriction. In the subsequent section, we introduce a multi-layered model with additive coupling bijections and triangular causal transformations that can express more complex distributions.

\subsection{One-layer linear flow}\label{sec_one_layer_flow}

In the simplest proof-of-concept model, where we assume we observe 2D Gaussian mixtures in $\w$ resulting from linear mechanisms, the transformation $T$ is then a block lower triangular matrix,
\begin{equation}\label{eq_linearA}
\textstyle
    {A} = \begin{pmatrix}
a_{11} & 0 \\
a_{21} & a_{22}
\end{pmatrix}\,.
\end{equation}
The log-likelihood then reduces to
$    \log p_\w(\w) = \log p_\z(\mathbf{A}^{-1}\w) +\sum_{i=1}^{2} \log |a_{ii}|$. 
We apply this simple model to simulated data with a one-dimensional cause below.

\begin{wrapfigure}{R}{0.45\textwidth}
  \begin{center}
    \includegraphics[width=0.45\textwidth]{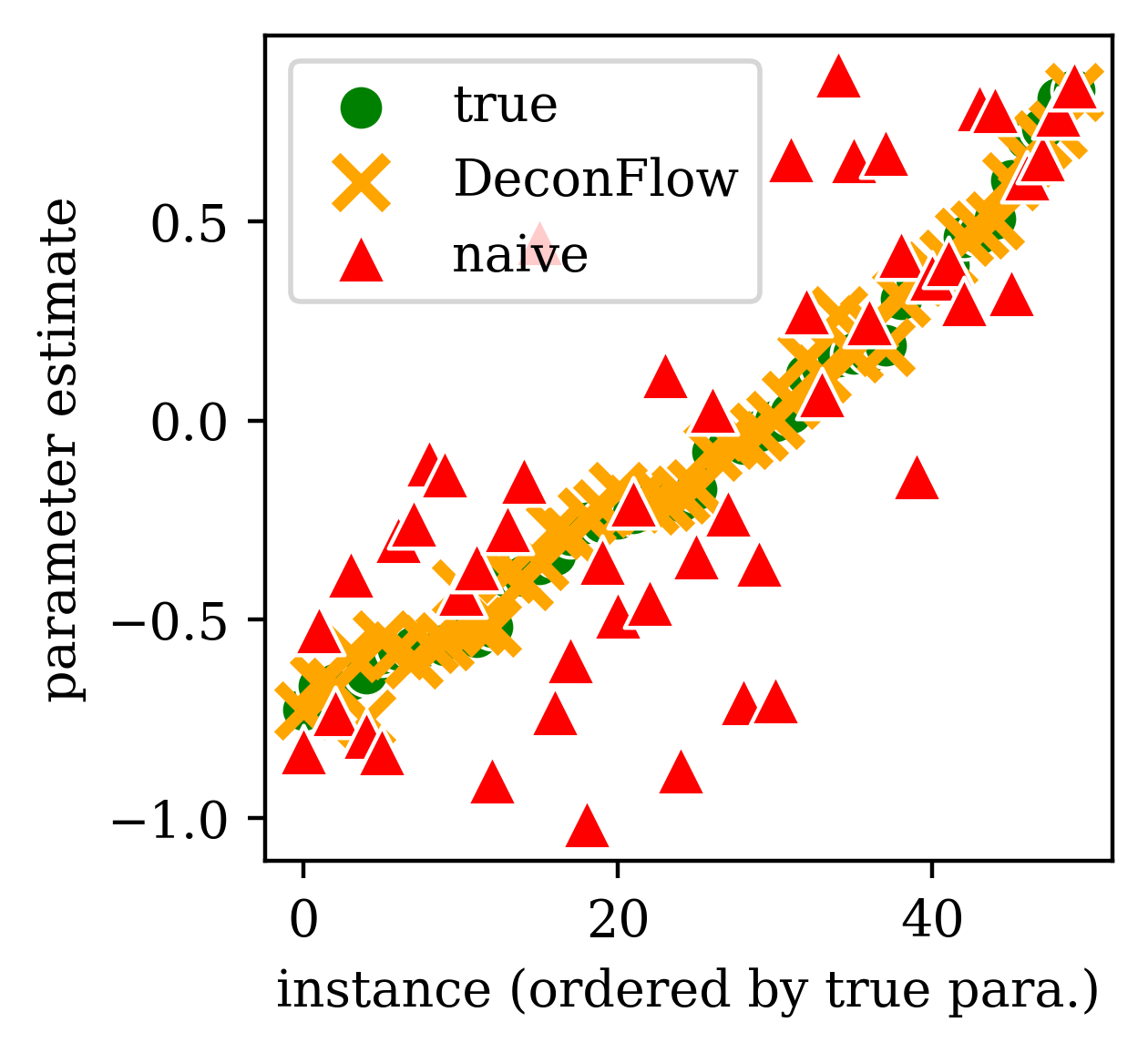}
  \end{center}
  \caption{With a one-dimensional cause and one-dimensional confounder, $m=n=1$, performance can be evaluated by comparing the $\DF$-adjusted slope parameter estimates (orange crosses) to the ground truth (green circles). In addition, we report the naive estimates that are obtained without addressing confounding (red triangles).}\label{fig_res_linear1d}
\end{wrapfigure}

\subsection{Additive coupling bijection}
To model more complicated distributions of $\w$, we propose a flow-based model where one transformation block is composed of an additive coupling layer \citep{dinh2014nice} and a causal tranformation akin to a masked autoregressive layer \citep{papamakarios2017masked}. Specifically, the transformations in one block are described in Algorithm \ref{alg_transformation_block}. Superscript $(l)$ denotes layer index, line \texttt{3} splits $\z_X$ into the first $n/2$ (rounded up if necessary) dimensions (subscript~$a$) and the remaining dimensions (subscript~$b$). The function $f_t$ in line \texttt{5} is parameterized by a neural network with ReLU activation function, the transformation matrix in line \texttt{7} has a partly-diagonal form,
\begin{equation*}
\mathbf{B} = \begin{bmatrix}
\text{diag}(\mathbf{a}) & \mathbf{0} \\
\mathbf{b} & b_{d,d}
\end{bmatrix}    
\end{equation*}
with $\mathbf{a} = \begin{bmatrix}
a_{1,1} & \cdots & a_{d-1,d-1}
\end{bmatrix}
$ and $\mathbf{b} = \begin{bmatrix}
a_{d,1} &  \cdots & a_{d,d-1}
\end{bmatrix}$, and $\mathbf{P}$ (only acting on $\z_X$, not $\z_Y$) in line \texttt{8} is a permutation matrix. By restricting $\mathbf{B}$ in this way and permuting only $\z_X$, we ensure that $\x$ influences $\y$ (but not vice versa), which reflects the assumed causal structure. Note that lines \texttt{5}~and~\texttt{6} differ from widely-used coupling bijections (which would additionally multiply $\z_b^{(l)}$ by a factor that is learned by $f_t$, as proposed in \cite{dinh16realNVP}) to ensure that the transformation is piecewise affine, which we require for identifiability. In practice, $N_B$ of such blocks are concatenated as depicted in Figure~\ref{fig_flow_structure}.

We can write the log-likelihood of $\w$ given these transformations as
\begin{equation}\label{eq_loglik_multilayer}
\textstyle
	\log p_\w(\w) = \log p_\z (\z^{(0)}) + \sum_{l=1}^{L}\sum_{i=1}^{d} \log |a_{ii}^{(l)}| 
\end{equation}
where $\z^{(0)} = \overline{T}\w$ with $\overline{T} = T_{(l=0)} \circ \hdots \circ T_{(l=L)}$ denoting the composition of the transformations described above (similarly for its inverse, $\overline{T}^{-1}$) and $p_\z$ being a Gaussian mixture model with diagonal covariances, as in Eq. \eqref{eq_GMM}. The transformation in line \texttt{6} is volume-preserving and has a unit Jacobian determinant. Therefore, its logarithm is equal to zero and vanishes in the log likelihood. Since the Jacobian of $\mathbf{B}$ is lower-triangular, its determinant is the product of the diagonal elements. We then optimize the log-likelihood in Eq. \eqref{eq_loglik_multilayer} using backpropagation.

\subsection{Closing the backdoor path through sampling}\label{sec_closing_backdoor}

Given our model structure, conditioning on $\Z_Y$ blocks the backdoor path between $\X$ and $\Y$. This motivates the following strategy to estimate $\mathbb{E}[Y|\doo(\X=\x)]$ from observed data. We transform the observed samples of $\w$ to $\z$ by inverting $\Psib$ using our trained model. We then sample $N_p$ times from the empirical distribution of $\tilde{\Z}_Y$ to compute 
\begin{equation}
\textstyle
    \overline{\w} = ({\x}, \overline{\y}) = \frac{1}{N_p}\sum_{\tilde{\z}_Y\sim P(\tilde{\Zb}_Y)}^{N_p} \overline{T}(\z_X, \tilde{\z}_Y)\,,
\end{equation}
where $\overline{\x}\!=\!\x$ because $\fb_X$ is  invertible. This yields the empirical counterpart to Eq. \eqref{eq_adjustment_main_text}, 
\begin{equation}\label{eq_estimated_do}
	\mathbb{E}[Y|\doo(\X=\x)] \approx \overline{\y} =: \hat{\theta}(\x).
\end{equation}

\section{Simulation Study}\label{sec_simulation}
\subsection{Data Generation}\label{sec_data_gen}
Given the generative model, we simulate data from a Generalized Additive Model (GAM, \cite{tibsh}) as follows. First, we randomly generate parameters of the joint distribution $P(L, Q)$ such that there is a correlation between $L$ and $Q$. Second, we generate $\Z_X \sim \mathcal{N}(\boldsymbol{\mu}_{h_X}, \Sigma_{h_X})$ and $Z_Y \sim \mathcal{N}(\mu_{h_Y}, \sigma^2_{h_X})$ 
where $\boldsymbol{\mu}_{h_X} \sim \mathcal{U}(1, 4)$ and $\mu_{h_Y} \sim \mathcal{U}(0, 1)$,  $\Sigma_{h_X} = \mathbf{I} \times 0.01$ and $\sigma^2_{h_X}=0.01$. We focus on the case with $m=1$, a scalar effect, in the simulation study.

\begin{wrapfigure}{L}{0.35\textwidth}
  \begin{center}
    \includegraphics[width=0.35\textwidth]{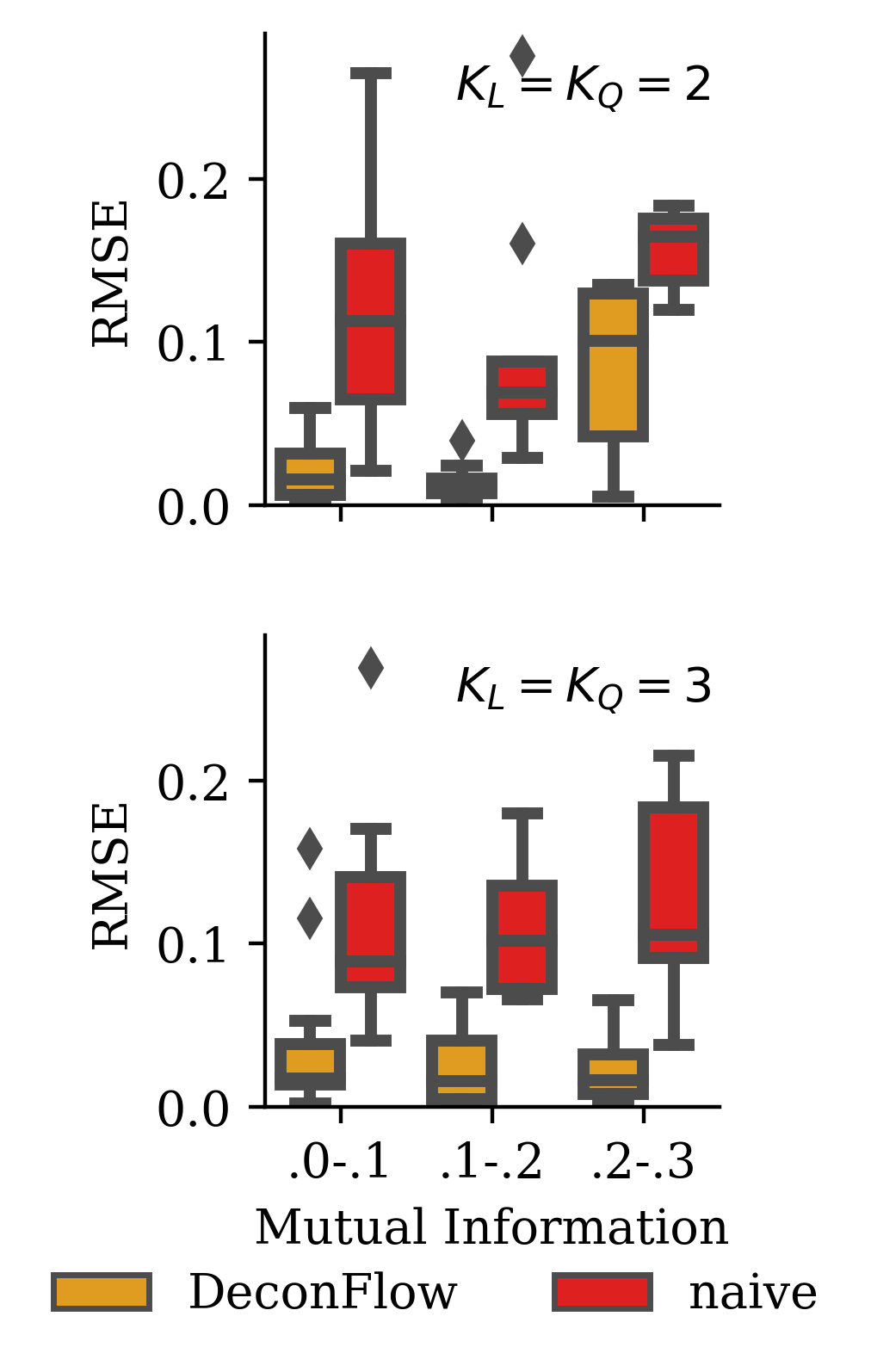}
  \end{center}
  \caption{See Section~\ref{sec_results_synthetic} for description.}\label{fig_nonlinear_5d}
\end{wrapfigure}

To generate $\X$ and $Y$, we then parameterize the influence of $\Z_X$ on $X$ and $Y$ as well as the influence of $Z_Y$ on $Y$ with random CDPA functions,
\begin{equation}\label{eq:additivesim}
    \X \!=\! \tau_1(\Z_X) \,, \,\mbox{ and } \, 
    Y\!=\! \beta \tau_2(\Z_X) + \tau_3(Z_Y) + \varepsilon\,,
\end{equation}
where $\beta$ is the true causal effect, and $\varepsilon\sim\mathcal{N}(0,0.01)$. Inspired by \cite{he2016deep}, the functions $\tau_1$, $\tau_2$, and $\tau_3$ are randomly initialized residual-flow type neural networks designed to generate an invertible piecewise affine transformation of data. The architecture consists of an initial linear layer, followed by a series of five ResNet blocks, and concludes with a final linear layer to produce the transformed output. Each ResNet block contains two linear layers with LeakyReLU activations and a skip connection, which adds the input of the block to its output. Note that the model class described in Eq.~\ref{eq:additivesim} is not covering the whole set of models considered in the theory. Notably, the effects of $\Z_X$ and $Z_Y$ on $Y$ are not required to be additive for our theoretical results to hold.

\textbf{Evaluation metric in linear case with $n=m=1$.} When $\tau_1$, $\tau_2$, and $\tau_3$ are identity mappings, we evaluate the ability of our method to deconfound by comparing the estimated slope parameter with the true causal effect $\beta$. In the linear case, the estimated parameter can be read off the estimate of the transformation matrix $A$ in \eqref{eq_linearA}: $\hat{\beta} = \frac{a_{21}}{a_{11}}$.

\textbf{Evaluation metric in the nonlinear case.} When $\tau_1$, $\tau_2$, and $\tau_3$ are random injective mappings, we evaluate the ground truth $\theta^*(\x):=\mathbb{E}[Y|do(\X=\x)]$ using Eq. \eqref{eq_adjustment_main_text} but for the ground truth model. We compare $\theta^*(\x)$ with the estimate defined in Eq. \eqref{eq_estimated_do}: \begin{equation}\label{eq_rmse}
\textstyle
	\text{RMSE}=\sqrt{
	\EE_{\x\sim P(\X)} \left[\left(\hat{\theta}(\x)-\theta^*(\x)\right)^2\right]}
\end{equation}

For comparison, we report a baseline RMSE that is obtained when the conditional density is erroneously used as a causal effect estimate:
\begin{equation}\label{eq_rmse_naive}
\textstyle
    \text{RMSE}_\text{naive}=\sqrt{
	\EE_{\x\sim P(\X)} \left[\left( \EE(Y|\x) -\theta^*(\x)\right)^2\right]}\;.
\end{equation}

\subsection{Results}\label{sec_results_synthetic}

\textbf{Linear one-layer, identity mapping.} First we generate 10,000 samples for the simple setting when $n=m=1$, and $\tau_1$, $\tau_2$, $\tau_3$ all being identity mappings, with $K_L=K_Q=2$, and apply the simple one-layer linear flow described in Section~\ref{sec_one_layer_flow}. In this case, the observed data \textit{is} a Gaussian mixture. Therefore, we have a setting in which the estimation procedure focuses solely on disentangling causal from confounded variation without additionally learning the mapping from observed data to a Gaussian mixture model. This setting serves as proof-of-concept of the deconfounding strategy. Results are shown in Figure~\ref{fig_res_linear1d}. It can be seen that the naive parameter estimates that are obtained by regressing observed $Y$ on observed $X$ are biased in arbitrary directions. Using $\DF$, we recover estimates of $\mathbb{E}[Y|\doo(X=x)]$, which we regress on $x$ to compute the deconfounded parameter estimates that almost perfectly match the ground truth.\footnote{Experiments are run on AWS Deep Learning AMI, with 36 vCPUs, runtime about 3 hours.}

\textbf{Nonlinear, invertible piecewise affine transformations.} Next we generate data with $n=5$, $m=1$ and $\tau_1$, $\tau_2$, $\tau_3$ random invertible piecewise affine functions (as described in Section~\ref{sec_data_gen}) and $K_L= K_Q=k$ for $k\in\{2, 3\}$, 10,000 observations. Figure~\ref{fig_nonlinear_5d} shows RMSE, see Eq. \eqref{eq_rmse}, and $\text{RMSE}_\text{naive}$, see Eq. \eqref{eq_rmse_naive}. The \textit{x}-axis shows mutual information between discrete variables $L$ and $Q$ as a measure for the strength of confounding. $\DF$ decreases the error incurred when estimating $\mathbb{E}[Y|\doo(\X=\x)]$ without observing the discrete confounder substantially. What we achieve here is the estimation of a nonlinear causal quantity, $\mathbb{E}[Y|\doo(\X=\x)]$, without observing the latent quantity that induces the discrepancy between it and $\mathbb{E}[Y|\x]$.\footnote{Experiments are run on AWS Deep Learning AMI, with 96 vCPUs, runtime about 10 hours.}

\section{Application}\label{sec_twins}
We use data on twin births in the USA collected around 1990, which has been used before by \cite{louizos2017causal} to illustrate causal inference methods. It contains measures of birth weight of newborn twins with about two dozen additional control covariates, such as parental education, number of prenatal visits, etc. for about 32,000 twins (and their parents). See Appendix \ref{app_twins} for a complete list of variables. The dataset lends itself to our setting because most of the variables are discrete and can serve as confounders. At the same time, some ordinal variables are also recorded. We choose as causes those ordinal variables so that we can approximate them with continuous variables by adding uniformly distributed noise. We do this because our model requires continuous cause variables and discrete confounding variables.

\begin{wrapfigure}{R}{0.28\textwidth}
  \begin{center}
    \includegraphics[width=0.28\textwidth]{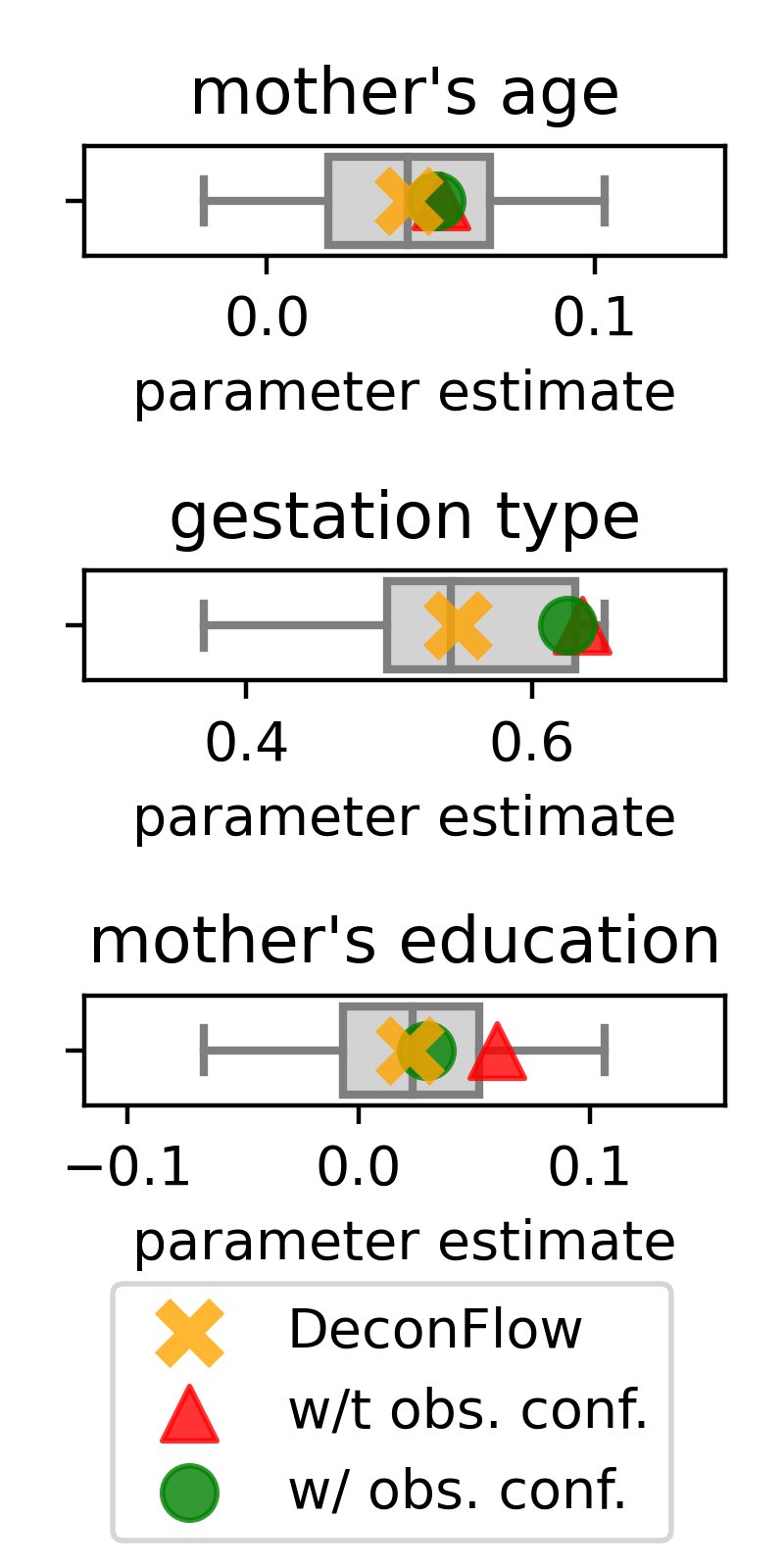}
  \end{center}
    \caption{See Section~\ref{sec_twins} for description.}
    \label{fig_twins_results}
\end{wrapfigure}

From the set of covariates $\{X_1, \dots, X_K\}$ we select the three ordinal variables that are directly related to the mother as observed causes: \emph{mother's age, gestation type}, and \emph{mother's education}, and denote them by $\X = \{X_1, X_2, X_3\}$. We use \emph{birth weight of the first-born twin} as target variable, $Y$, and treat all remaining covariates as confounders, denoted by $\V = \{X_4, \dots, X_K\}$. This allows us to estimate ``true'' causal effects when we treat the confounders as observed, and test whether $\DF$ can recover these given only the data about $\X$ and $Y$.

Predicting $Y$ using least-squares regression, we estimate the parameter vector for $\X$ once when controlling for $\V$ (denoted $\mathbf{\beta}^*$) and once when not  controlling for $\V$ (denoted $\hat{\mathbf{\beta}}$). We run our deconfounding approach as described in Section~\ref{sec_closing_backdoor} using only $\{\X, Y\}$, which yields our estimate of $\hat{\theta}(\x) = \mathbb{E}[Y|\doo(\X=\x)]$. We then regress $\hat{\theta}(\x)$ on $\X$ to estimate our debiased parameter vector, $\tilde{\mathbf{\beta}}$. We can evaluate whether our method can account for the confounders $\V$ (that are unobserved from its perspective) by comparing $\mathbf{\beta}^*$ with $\hat{\mathbf{\beta}}$ and~$\tilde{\mathbf{\beta}}$.

We run $\DF$ for multiple seeds and hyperparameters. In Figure~\ref{fig_twins_results}, for each of the three cause variables (\emph{mother's age, gestation type}, and \emph{mother's education}), we report \textit{i}) the slope parameter of that cause variable in a regression of $Y$ on the three causes (red triangle), \textit{ii}) the slope parameter of that cause variable in a regression of $Y$ on the three causes and the observed confounders (green dot), \textit{iii}) the average slope parameter of that cause in a regression of the $\DF$-adjusted target variable $\tilde{Y}$ on the three causes for 32 runs of $\DF$ (orange cross), as well as a boxplot of the underlying distribution of this parameter. For causes \emph{mother's age} and \emph{mother's education}, we observe that our method yields mean parameter estimates that are closer to $\mathbf{\beta}^*$ than $\hat{\mathbf{\beta}}$. For \emph{gestation type}, we find $\tilde{\mathbf{\beta}}$ to be lower than both $\mathbf{\beta}^*$ and $\hat{\mathbf{\beta}}$.

While we consider similar $\mathbf{\beta}^*$ and $\tilde{\mathbf{\beta}}$ as evidence that our method accounts for $\V$ without observing it, we stress that $\mathbf{\beta}^*$ might in fact differ from the true parameter vector because of residual confounding that is not captured by $\V$. That is, a discrepancy between $\mathbf{\beta}^*$ and $\tilde{\mathbf{\beta}}$ might indicate the existence of additional confounders unmeasured in the dataset, rather than a shortcoming of our method. For instance, the discrepancey between $\tilde{\mathbf{\beta}}$ and $\mathbf{\beta}^*$ for \emph{gestation type} could be due to additional unmeasured confounders.

\section{Discussion}\label{sec_discussion}
While there is a large literature on using measured confounders to deconfound causal effect estimates (see e.g. \cite{chernozhukov2018double}), or to gauge the sensitivity to unmeasured confounders by benchmarking against \textit{measured} confounders in treatment effect estimation \citep{cinelli20sensitivity} or policy learning \citep{kallus2021minimax, marmarelis2024policy}, work on accounting for unmeasured confounders without such benchmarks is scarce. In the following we provide a brief overview of related work that addresses unmeasured confounding without access to observed confounders.

One way to tackle unmeasured confounding is to make assumptions on the independence of causal mechanisms (ICM) \cite{peters2017elements, janzing2010causal}. For instance, \cite{JS, JSb} formalize ICM in multivariate linear models to estimate a degree of confounding. ICM can also be seen as motivating additive noise models as used in \cite{janzing2012identifying}, which is similar to our approach in the sense that a latent confounder is learned from observed variables. However, this method does not allow for both a causal \textit{and} a confounding effect between the two variables.

Even without implicit or explicit motivation through ICM, restricing model classes can help to address unmeasured confounding. For instance, assuming linear relations and non-Gaussian variables yields identifiability of a number of causal properties \cite{shimizu2006linear}. In this model class, \cite{hoyer2008estimation} show how independent component analysis (ICA) with an overcomplete basis (recovering more source variables than there are observed signals), can help to theoretically identify, up to some remaining ambiguity, the latent confounder and causal effect. However, practical algorithms that reliably estimate an overcomplete basis are lacking and require additional assumptions (such as sparsity of the mixing matrix). Methods for (nonlinear) ICA with equal number of sources and signals include e.g. \cite{khemakhem20VAE_ICA, hyvarinen2017nonlinear} but these require observed auxiliary information (such as environment variables) or assumptions like ICM \cite{gresele2021independent}. None of these methods can address unmeasured confounding in a principled and practical way, which is the goal of our proposed method.

\textbf{Limitations.} As all causal inference techniques, the proposed methodology relies on assumptions that, if not satisfied, can cast doubt on causal effect estimates that are produced using the method. While the discrete nature of the confounding we are considering has applications in a variety of domains (e.g., controlling for batch effects in high-throughput sequencing data \cite{leek2010tackling}), it is a substantial assumption that needs to be taken into account by practitioners. Furthermore, we restrict the latent variables to follow a Gaussian mixture model and the function mapping from latent to observed variables to be piecewise affine and injective. While this is a very flexible model class, how our causal effect identification result generalizes to the case where the ground truth model does not strictly belong to this class remains an open question.

\section{Conclusion}

We propose a method to address unmeasured discrete confounding in nonlinear cause-effect models. By mapping a confounded causal model to an equivalent latent variable model, we can leverage identifiability results in the literature on such models. We demonstrate that, under specific assumptions, it is possible to identify causal effects despite the presence of unmeasured confounders. We introduce a flow-based algorithm that can correct for this type of unmeasured confounding. The empirical results on both synthetic and real-world data provide evidence of the effectiveness of our approach.

As such, this work is an effort at building a bridge between the literature on causal inference that uses constraints on function classes and deep latent variable models. The usefulness of deep latent variable models have successfully been shown in a variety of applications and has spurned an interested in analyzing their identifiability properties, whose connections to causal inference problems we explore here.

Future work may investigate how the proposed strategy can be extended to more complex causal graphs, other model classes, and other estimable causal quantities such as counterfactuals.

\bibliographystyle{plain}

\bibliography{NeurIPS2024}

\appendix
\newpage

\section*{Appendices}\label{app}

\section{Proof of main text results}\label{app_proofs}
\identif*
\begin{proof}
    \textbf{Step 1}: \textit{Affine identifiability.}
    
    The above model can be rewritten as a piecewise affine injective mapping 
    \begin{align}
    \Psib:&\,\quad\Zcal\to\Xcal \times \Ycal\,,\\
    &\begin{bmatrix}
    \zb_X\\
    \zb_Y
    \end{bmatrix}
     \mapsto 
     \begin{bmatrix}
    \fb_X( \zb_X)\\
     \fb_Y(\fb_X(\zb_X),\zb_Y)
    \end{bmatrix}\,.
    \end{align}
    Therefore we get affine identifiability from \cite[Theorem 3.2]{kivva2022identifiability}.
    
    \textbf{Step 2}: \textit{Form restriction on the affine transformation due to partial observation.}\footnote{Restriction on the ambiguity that results because we only recover $g, \mathcal{Z}$ up to affine transformation. The point here is that it is a very special ambiguity, namely one where $\mathbf{A}$ is diagonal.}
    Assume another solution $\Tilde{f}$, it can also be rewritten as an injective mapping
 \begin{align}
    \tilde{\Psib}:&\,\quad\Zcal\to\Xcal \times \Ycal\,,\\
    &\begin{bmatrix}
    \zb_X\\
    \zb_Y
    \end{bmatrix}
     \mapsto 
     \begin{bmatrix}
     \Tilde{\fb_X}(\zb_X)\\
     \tilde{\fb_Y}(\fb_X(\zb_X),\zb_Y)
    \end{bmatrix}\,.
    \end{align}
    By affine identifiability, $\tilde{\Psib}^{-1}\circ \Psib$ is an affine map $\zb\mapsto A\zb+\bb$. From the above we deduce that\footnote{This is because $\Psib$ is lower triangular, therefore $\tilde{\Psib}$ is lower triangular, therefore $\tilde{\Psib}^{-1}$ is lower triangular, and therefore $\tilde{\Psib}^{-1}\circ \Psib$ is lower triangular.}
     \begin{align}
   A&=
     \begin{bmatrix}
     T 
     & 0\\
     U & S
    \end{bmatrix}\,.
    \end{align}
    with $U$ an $m\times n$ row vector, $T$ an invertible matrix and $S$ a non-vanishing scalar (due to invertibility of both functions). 
    
    \textbf{Step 3}: \textit{Further form restriction due to non-degeneracy of intra-mixture component covariances. }
    Let us consider the ground truth distribution of $\Z$: due to Assumption.~\ref{assum:nondeg} it is a Gaussian mixture, whose mixture components are indexed by $\{(l,q)\}_{l=1..K_L;q=1..K_Q}$ and whose associated covariances are of block diagonal of the form
    \[
    \Sigma_{l,q} = \begin{bmatrix}
        \Sigma_l^X & \boldsymbol{0}\\
        \boldsymbol{0} & \Sigma_q^Y
    \end{bmatrix} \,.
    \]
    Moreover, this is the same for the retrieved latent $\tilde{\Z}$, up a permutation of indices $(l,q)\mapsto \sigma(l,q)$ and the affine transformation introduced above (e.g. using Theorem C.2 in \cite{kivva2022identifiability}, stating that the mixture components are identified up to a permutation and affine transformation).  
      As a consequence we get, for any index $(l,q)$, that the corresponding mixture component covariance $\widetilde{\Sigma}_{\sigma(l,q)}$ correspond $\Sigma_{l,q}$ after linear transformation of the Gaussian distribution by matrix $A$, i.e.  
     \begin{align}
   \widetilde{\Sigma}_{\sigma(l,q)}= A \Sigma_{l,q} A^\top &=
     \begin{bmatrix}
     T 
     & 0\\
     U & S
    \end{bmatrix} \begin{bmatrix}
     \Sigma_l^X 
     & 0\\
     0 & \Sigma_q^Y
    \end{bmatrix}
    \begin{bmatrix}
     T^\top 
     & U^\top\\
     0 & S^\top
    \end{bmatrix}\\
    &=
     \begin{bmatrix}
     T 
     & 0\\
     U & S
    \end{bmatrix} \begin{bmatrix}
     \Sigma_l^XT^\top  
     & \Sigma_l^{X}U^\top\\
     0 & \Sigma_q^{Y} S^\top
    \end{bmatrix}\\
    &= \begin{bmatrix}
     T \Sigma_l^{X}T^\top  
     & T\Sigma_l^{X}U^\top\\
     U \Sigma_l^{X}T^\top   & S \Sigma_q^{Y} S^\top + U\Sigma_l^{X}U^\top
    \end{bmatrix}
    \,.
    \end{align}
      where the off diagonal blocks must again be equal to zero by Assumption~\ref{assum:nondeg} applied to the covariance of the mixture component of the obtained solution $\widetilde{\Sigma}_{\sigma(l,q)}$. 
Exploiting this assumption further, let us choose $l$ such that $\Sigma_l^X$ is positive definite. In that case, we can write for the off-diagonal block
      \begin{align}
          U \Sigma_l^{X} T^\top &= 0\\
          U \Sigma_l^{X}  &= 0 \; \text{because $T^\top$ is invertible}\\
          U &= 0 \; \text{because $\Sigma_l^{X}$ is positive definite and therefore invertible.}
      \end{align}

Consequently,

      \begin{align}
   A&=
     \begin{bmatrix}
     T & 0\\
     0 & S
    \end{bmatrix}\,,
    \end{align}
    which entails identifiability up to scalar affine reparametrization of $Z_2$ and affine invertible transformation of $Z_1$.

    More precisely, for all $\zb_1,\zb_2$, the composition of $\tilde{\Psib}^{-1}$ with $\Psib$ is ambiguous up to a diagonal affine transformation: 
    \[
\begin{bmatrix}
        \tilde{\zb}_X\\
        \tilde{\zb}_Y
    \end{bmatrix}=    \tilde{\Psib}^{-1}\circ \Psib(\zb_X,\zb_Y)=\begin{bmatrix}
        T\zb_X+\bb_1\\
        S \zb_Y+\bb_2
    \end{bmatrix}
    \]
    Leading to 
    \[
    \Psib(\zb_X,\zb_Y)=\tilde{\Psib}(
        T\zb_X+\bb_X,
        S z_Y+\bb_Y)
    \]
    For the $\Xb$ component this gives
    \[
\fb_X(\zb_X) = \tilde{\fb}_X(T\zb_X+\bb_X) 
    \]
    such that
    \[
\fb_X^{-1}(\xb) = T^{-1} \left(\tilde{\fb}_X^{-1}(\xb)-\bb_X\right) 
    \]
    because $(f \circ g)^{-1} = g^{-1} \circ f^{-1} $. 
    And for the $\Yb$ component this gives,
    \[
    \fb_Y(\fb_X(\zb_X),\zb_Y) = \tilde{\fb}_Y(\tilde{\fb}_X(T\zb_X+\bb_X), S \zb_Y +\bb_Y)
    \]
    Finally we get the following relation for the causal mechanism
    \[
    \fb_Y(\xb,\zb_Y)=\tilde{\fb}_Y(\fb_X(\zb_X),S \zb_Y+\bb_Y)=\tilde{\fb}_Y(\xb,S \zb_Y+\bb_Y)
    \]    
\end{proof}

\identadjust*
\begin{proof}
Consider a given $
\xb$ in the support of $p(\Xb)$, the above backdoor adjustment require $p(\yb|\Xb=\xb,\zb_Y)$ to be well defined for almost any $\zb_Y$. 
Given our generative model of Section~\ref{sec_data_gen}, this amounts to having $\fb$ unambiguously defined for almost any $\zb_Y$. As $\fb_Y$ is only unambiguously identified on the support of the observational distribution $p(\xb,\zb_Y)$, it is necessary and sufficient to have strict positivity of $p(\xb,\zb_Y)$ for almost all $\zb_Y$. The adjustment formula using $\Zb_Y$ is given by
     \begin{equation*}
\EE \left[\Yb|do(\Xb=\xb) \right]=\EE_{\Zb_2\sim P(\Zb_Y)}\left[\fb(\xb,\Zb_Y)\right]\end{equation*}
Using Theorem~\ref{thm:identif} we can rewrite the expression of function $\fb$ such that
\[
\EE \left[\Yb|do(\Xb=\xb) \right]=\EE_{\Zb_Y\sim P(\Zb_Y)}\left[\tilde{\fb}_Y(\xb,S \Zb_Y+\bb)\right]\,.
\]
Moreover, we can replace the (unknown) latent variable distribution $P(\Z_2)$ with the estimated latent variable distribution $P(\tilde{\Z}_2)$ to obtain the result
\begin{equation}\label{eq_sampling_tildeZ2}
    \EE \left[\Yb|do(\Xb=\xb) \right]=\EE_{\tilde{\Zb}_Y\sim P(\tilde{\Zb}_Y)}\left[\tilde{\fb}_Y(\xb,\tilde{\Zb}_Y)\right]\,.
\end{equation}
\end{proof}

\sufficident*
\begin{proof}
    As $p(\xb,\zb_Y)$ is the pushforward of $p(\zb_X,\zb_Y)$ by an invertible, continuous, differentiable almost everywhere, function $\Psib$ defined in the proof of Theorem~\ref{thm:identif}. Therefore,  $p(\xb,\zb_Y)$ is strictly positive if and only if $p(\zb_X=\fb_X^{-1}(\xb),\zb_Y)$ is strictly positive. Since $p(\zb_X,\zb_Y)$ is a Gaussian mixture, it is sufficient to have at least one non-degenearate mixture component occurring with non-zero probability strict positivity (see Assumption~\ref{assum:nondeg})..
\end{proof}

\section{Structural causal models}\label{app_SCM}
Causal dependencies between variables can be described using \textit{Structural Causal Models} (SCM)~\citep{Pearl_2009}.
\begin{definition}[SCM]\label{def:SCM}
    An $n$-variable SCM is a triplet $\Mcal=(\mathcal{G},\mathbb{S},P_\Ub)$ consisting of:
    \begin{itemize}
        \item a directed acyclic graph $\mathcal{G}$ with $n$ vertices,
        \item a set $\mathbb{S}=\{\Vb_j \coloneqq \fb_j(\textbf{Pa}_j,\Z_j), j=1,\dots,n\}$ of structural equations, 
        where $\textbf{Pa}_j$ are the variables indexed by the set of parents of vertex $j$ in $\mathcal{G}$,
        \item a joint distribution $P_\Z$ over the exogenous variables $\{\Z_j\}_{j\leq n}$.
    \end{itemize} 
\end{definition}
 
Due to the directed acyclic structure of $\Gcal$, for each value of the exogenous variables, $\mathbb{S}$ leads to a unique solution for the vector of so-called endogenous variables $\Vb=[\Vb_1, \dots, \Vb_n]^\top$, such that the distribution $P_{\Z}$ entails a well-defined joint distribution over the endogenous variables $P(\Vb)$. For the purpose of the present work, we adopt a very general setting by: (1) not enforcing joint independence between the exogenous variables, allowing them to encode hidden confounding, (2) allowing endogenous and exogenous variable to be vector-valued. A given set of random variables, there may be described by different SCMs, e.g. by making different choices of grouping components in vector variables $\Vb_k$, or by choosing which will appear as exogenous or endogenous variables. We may switch between different such choices, provided those choices make a equivalent predictions regarding interventions that we introduce next. 

We will consider $do$-interventions in SCMs involve replacing one or more structural equation by a constant and modifying $\Gcal$ accordingly such that parents of the intervened equations are removed. 
An intervention transforms the original model $\Mcal =(\Gcal,\mathbb{S},P_\Z)$ into an intervened model $\Mcal^{do(\Vb_k=\vb_k)} =\left(\mathcal{G}^{do(\Vb_k=\vb_k)},\mathbb{S}^{do(\Vb_k=\vb_k)},P_\Z^{do(\Vb_k=\vb_k)}\right)$, where $\vb_k$ is the constant parameterizing the intervention.

\subsection{Unmeasured confounding and backdoor criterion}
In the standard setting of causal effect estimation, one focuses on a graph comprising a pair of endogenous variables $(\Xb,\Yb)$ such that $\Gcal$ contains the edge $\Xb\to\Yb$. Hidden counfounding can then be encoded by non-independence of the respective exogenous variables $\Z_\Xb$ and $\Z_\Yb$ of these nodes, which we represent as a dashed bidirectional arrow in Figure~\ref{fig_models}a. Our framework amounts to constraining the structure of this hidden confounding, which is assumed to be representable as an hidden discrete common cause of two hidden latent variables $\Zb_X$ and $\Zb_Y$, as described by the causal diagram of Figure~\ref{fig_models}b, which does not have any dependence between exogenous variables of the nodes $\Xb$ and $\Yb$, because confounding is now explicitly represented by a common cause $H$. The additional variables appearing in this new graph, if they were to be observed, could be used to estimate the interventional probability $P(\Yb|\doo(\Xb=\xb))$ because they satisfied the so-called backdoor criterion \cite{Pearl_2009}: they block all backdoor paths between $\Xb$ and $\Yb$, i.e. those going through a parent of $\Xb$. Although latent variable are unobserved, additional assumption may allow to identify them from observational data. In particular, one way is to formulate the observations as a function of the latents, which can be done by introducing an invertible mapping $\phib:\Zb_X\to \X$, leading to the causal diagram of Figure~\ref{fig_models}c. 

We will focus on a case where it can be shown that we can infer and use $\Zb_Y$ as a backdoor adjustment variable, which leads to the following formula for the interventional distribution
\[
P(\Yb|\doo(\Xb))]= \int  P(\yb|\x,\zb_y) p(\zb_y)d\zb_y\,.
\]


\section{Twins dataset}\label{app_twins}

The remaining confouding variables are: `risk factor, Lung', `risk factor Hemoglobinopathy',
 `risk factor, Incompetent cervix',
 `mom place of birth',
 `race of child',
 `total number of births before twins',
 `trimester prenatal care begun, 4 is none',
 `number of live births before twins',
 `married',
 `risk factor, Anemia',
 `risk factor, Hypertension, chronic',
 `risk factor, RH sensitization',
 `num of cigarettes /day, quantiled',
 `risk factor, tobacco use',
 `education category',
 `state of occurence FIPB',
 `medical person attending birth',
 `quintile number of prenatal visits',
 `US census region of mplbir',
 `dad race',
 `place of delivery',
 `risk factor, Renal disease',
 `mom race',
 `risk factor, Cardiac',
 `US census region of stoccfipb',
 `risk factor, Previous infant 4000+ grams',
 `US census region of brstate',
 `birth month Jan-Dec',
 `risk factor, Eclampsia',
 `risk factor, Other Medical Risk Factors',
 `octile age of father',
 `risk factor, alcohol use',
 `dad hispanic',
 `num of drinks /week, quantiled',
 `risk factor, Herpes',
 `mom hispanic',
 `risk factor, Hypertension, preqnancy-associated',
 `state of residence NCHS',
 `risk factor, Uterine bleeding',
 `risk factor, Diabetes',
 `sex of child',
 `risk factor Hvdramnios/Oliqohvdramnios',
 `risk factor, Previos pre-term or small',
 `adequacy of care'.

\newpage
\section*{NeurIPS Paper Checklist}

\begin{enumerate}

\item {\bf Claims}
    \item[] Question: Do the main claims made in the abstract and introduction accurately reflect the paper's contributions and scope?
    \item[] Answer: \answerYes{} 
    \item[] Justification: Yes, claims are accurate.

\item {\bf Limitations}
    \item[] Question: Does the paper discuss the limitations of the work performed by the authors?
    \item[] Answer: \answerYes{} 
    \item[] Justification: Limitations are discussed in Section \ref{sec_discussion} with a separate `Limitations' section.

\item {\bf Theory Assumptions and Proofs}
    \item[] Question: For each theoretical result, does the paper provide the full set of assumptions and a complete (and correct) proof?
    \item[] Answer: \answerYes{} 
    \item[] Justification: Proofs to all results are given in the Appendix \ref{app_proofs}. In addition, a proof sketch for the main result is given in the main text.

    \item {\bf Experimental Result Reproducibility}
    \item[] Question: Does the paper fully disclose all the information needed to reproduce the main experimental results of the paper to the extent that it affects the main claims and/or conclusions of the paper (regardless of whether the code and data are provided or not)?
    \item[] Answer: \answerYes{} 
    \item[] Justification: Section \ref{sec_simulation} contains information about all the parameters used in the simulation results. Code to generate the synthetic data and to implement the method is provided in an anonymized zip file in the Supplementary Material.

\item {\bf Open access to data and code}
    \item[] Question: Does the paper provide open access to the data and code, with sufficient instructions to faithfully reproduce the main experimental results, as described in supplemental material?
    \item[] Answer: \answerYes{} 
    \item[] Justification: Yes, link to code will be provided on the first page conditional on acceptance. It is provided with the submission as a zip file in the Supplementary Material to guarantee anonymity.
    \item[] Guidelines:
    \begin{itemize}
        \item The answer NA means that paper does not include experiments requiring code.
        \item Please see the NeurIPS code and data submission guidelines 
        \item While we encourage the release of code and data, we understand that this might not be possible, so “No” is an acceptable answer. Papers cannot be rejected simply for not including code, unless this is central to the contribution (e.g., for a new open-source benchmark).
        \item The instructions should contain the exact command and environment needed to run to reproduce the results. See the NeurIPS code and data submission guidelines 
        \item The authors should provide instructions on data access and preparation, including how to access the raw data, preprocessed data, intermediate data, and generated data, etc.
        \item The authors should provide scripts to reproduce all experimental results for the new proposed method and baselines. If only a subset of experiments are reproducible, they should state which ones are omitted from the script and why.
        \item At submission time, to preserve anonymity, the authors should release anonymized versions (if applicable).
        \item Providing as much information as possible in supplemental material (appended to the paper) is recommended, but including URLs to data and code is permitted.
    \end{itemize}

\item {\bf Experimental Setting/Details}
    \item[] Question: Does the paper specify all the training and test details (e.g., data splits, hyperparameters, how they were chosen, type of optimizer, etc.) necessary to understand the results?
    \item[] Answer: \answerYes{} 
    \item[] Justification: These details can be seen in the provided code.

\item {\bf Experiment Statistical Significance}
    \item[] Question: Does the paper report error bars suitably and correctly defined or other appropriate information about the statistical significance of the experiments?
    \item[] Answer: \answerNo{} 
    \item[] Justification: While no formal error bars are shown, we show results for a number of draws from the data generating process and report the distribution of results in the synthetic data experiments. In the real-world data application, we show results for a number of hyperparameter choices and seeds. 

\item {\bf Experiments Compute Resources}
    \item[] Question: For each experiment, does the paper provide sufficient information on the computer resources (type of compute workers, memory, time of execution) needed to reproduce the experiments?
    \item[] Answer: \answerYes{} 
    \item[] Justification: That information is provided in the relevant Sections of the paper.
    
\item {\bf Code Of Ethics}
    \item[] Question: Does the research conducted in the paper conform, in every respect, with the NeurIPS Code of Ethics ?
    \item[] Answer: \answerYes{} 
    \item[] Justification: Yes, the research conducted here conforms the NeurIPS Code of Ethics.

\item {\bf Broader Impacts}
    \item[] Question: Does the paper discuss both potential positive societal impacts and negative societal impacts of the work performed?
    \item[] Answer: \answerNo{} 
    \item[] Justification: While the paper does not discuss negative societal impacts, it emphasizes that the results for the proposed causal inference technique rest on assumptions that need to be fulfilled for the method to work as expected.
    
\item {\bf Safeguards}
    \item[] Question: Does the paper describe safeguards that have been put in place for responsible release of data or models that have a high risk for misuse (e.g., pretrained language models, image generators, or scraped datasets)?
    \item[] Answer: \answerNo{} 
    \item[] Justification: No risk for misuse.

\item {\bf Licenses for existing assets}
    \item[] Question: Are the creators or original owners of assets (e.g., code, data, models), used in the paper, properly credited and are the license and terms of use explicitly mentioned and properly respected?
    \item[] Answer: \answerYes{} 
    \item[] Justification: The use of code by other researchers is acknowledged.
    \item[] Guidelines:
    \begin{itemize}
        \item The answer NA means that the paper does not use existing assets.
        \item The authors should cite the original paper that produced the code package or dataset.
        \item The authors should state which version of the asset is used and, if possible, include a URL.
        \item The name of the license (e.g., CC-BY 4.0) should be included for each asset.
        \item For scraped data from a particular source (e.g., website), the copyright and terms of service of that source should be provided.
        \item If assets are released, the license, copyright information, and terms of use in the package should be provided. For popular datasets, 
        has curated licenses for some datasets. Their licensing guide can help determine the license of a dataset.
        \item For existing datasets that are re-packaged, both the original license and the license of the derived asset (if it has changed) should be provided.
        \item If this information is not available online, the authors are encouraged to reach out to the asset's creators.
    \end{itemize}

\item {\bf New Assets}
    \item[] Question: Are new assets introduced in the paper well documented and is the documentation provided alongside the assets?
    \item[] Answer: \answerYes{} 
    \item[] Justification: Code with documentation is provided.

\item {\bf Crowdsourcing and Research with Human Subjects}
    \item[] Question: For crowdsourcing experiments and research with human subjects, does the paper include the full text of instructions given to participants and screenshots, if applicable, as well as details about compensation (if any)? 
    \item[] Answer: \answerNA{} 
    \item[] Justification: Not applicable.

\item {\bf Institutional Review Board (IRB) Approvals or Equivalent for Research with Human Subjects}
    \item[] Question: Does the paper describe potential risks incurred by study participants, whether such risks were disclosed to the subjects, and whether Institutional Review Board (IRB) approvals (or an equivalent approval/review based on the requirements of your country or institution) were obtained?
    \item[] Answer: \answerNA{} 
    \item[] Justification: Not applicable.

\end{enumerate}

\end{document}